\documentclass[letterpaper,11pt]{article}
\usepackage{booktabs} 
\usepackage[margin=1in]{geometry}

\usepackage[ruled, noend]{algorithm2e}
\SetAlFnt{\small}
\SetAlCapFnt{\small}
\SetAlCapNameFnt{\small}
\SetAlCapHSkip{0pt}
\IncMargin{-\parindent}

\SetCommentSty{mycommfont}

\usepackage[utf8]{inputenc}
\usepackage{amssymb}
\usepackage[margin=1in]{geometry}
\usepackage{amsthm}
\usepackage{thmtools}
\usepackage{thm-restate}

\usepackage{amsmath}
\usepackage{comment}
 
\usepackage{amssymb}
\usepackage{tikz}
\usepackage{mathtools}
\usepackage{bbm}

\usepackage{dsfont}
\usepackage{graphicx}
\usepackage{multirow}
\usepackage{color}
\usepackage{xcolor}
\usepackage{natbib}
\usepackage{enumitem}

\usepackage{hyperref}
\usepackage{graphicx}
\usepackage{subfigure}
\hypersetup{%
   breaklinks,%
   colorlinks=true,%
   linkcolor=black,%
   urlcolor=black,%
   citecolor=[rgb]{0,0,0.45}
}

\theoremstyle{plain}
\newtheorem{theorem}{Theorem}[section]

\newtheorem{lemma}[theorem]{Lemma}

\newtheorem{claim}[theorem]{Claim}

\newtheorem{definition}[theorem]{Definition}

\newtheorem{remark}[theorem]{Remark}
\theoremstyle{definition}
\newtheorem{example}[theorem]{Example}
\newtheorem{observation}[theorem]{Observation}

\DeclareMathOperator*{\argmin}{argmin}
\DeclareMathOperator*{\argmax}{argmax}

\newcommand{\reals}{\mathbb{R}}
\newcommand{\br}{\text{BR}}

\usepackage{color-edits}[showdeletions]

\setcitestyle{authoryear}

\author{Lee Cohen\thanks{Stanford. Email: \href{mailto:leeco@stanford.edu}{leeco@stanford.edu}} \and Saeed Sharifi-Malvajerdi\thanks{Toyota Technological Institute at Chicago (TTIC). Email: \href{mailto:vakilian@ttic.edu}{saeed@ttic.edu}} \and Kevin Stangl\thanks{Toyota Technological Institute at Chicago (TTIC). Email: \href{mailto:kevin@ttic.edu}{kevin@ttic.edu}} \and Ali Vakilian\thanks{Toyota Technological Institute at Chicago (TTIC). Email: \href{mailto:vakilian@ttic.edu}{vakilian@ttic.edu}} \and Juba Ziani \thanks{Georgia Institute of Technology. Email: \href{mailto:jziani3@gatech.edu}{jziani3@gatech.edu}}}

\date{}

\title{Bayesian Strategic Classification}

\begin{document}

\maketitle

\begin{abstract}
In strategic classification, agents modify their features, at a cost, to ideally obtain a positive classification outcome from the learner's classifier. The typical response of the learner is to carefully modify their classifier to be robust to such strategic behavior. When reasoning about agent manipulations, most papers that study strategic classification rely on the following strong assumption: agents \emph{fully} know the exact parameters of the deployed classifier by the learner. This often is an unrealistic assumption when using complex or proprietary machine learning techniques in real-world and high-stakes scoring and prediction tasks.

We initiate the study of \emph{partial} information release by the learner in strategic classification. We move away from the traditional assumption that agents have full knowledge of the classifier. Instead, we consider agents that have a common \emph{distributional prior} on which classifier the learner is using, and manipulate to maximize their expected utility according to this prior. The learner in our model can reveal truthful, yet not necessarily complete, information about the deployed classifier to the agents. The learner's goal is to release \emph{just enough} information about the classifier to maximize accuracy. We show how such partial information release can, counter-intuitively, benefit the learner's accuracy, despite increasing agents' abilities to manipulate. Releasing \emph{some} information about the classifier can help qualified agents pass the classifier, while not releasing enough information for unqualified agents to do so. 

We show that while it is intractable to compute the best response of an agent in the general case, there exist \emph{oracle-efficient} algorithms that can solve the best response of the agents when the learner's hypothesis class is the class of (low-dimensional) linear classifiers, or when the agents' cost function satisfies a natural notion of ``submodularity'' as we define. We then turn our attention to the learner's optimization problem and provide both positive and negative results on the algorithmic problem of how much information the learner should release about the classifier to maximize their expected accuracy, in settings where each agent's qualification level for a certain task can be captured using a real-valued number. 
\end{abstract}

\newpage
\section{Introduction}
Traditional machine learning critically relies on the assumption that the training data is representative of the unseen instances a learner faces at test time. Yet, in many real-life situations, this assumption fails. In particular, a critical failure point occurs when individuals (agents) manipulate decision-making algorithms for personal advantage, often at a cost. A typical example of such manipulations or \emph{strategic behavior} is seen in loan applications or credit scoring: for example, an individual may open new credit card accounts to lower their credit utilization and increase their credit score artificially. 
In the context of job interviews, a candidate can spend time and effort to memorize solutions to common interview questions and potentially look more qualified than they are at the time of an interview. A student might cram to pass an exam this way without actually understanding or improving their knowledge of the subject.

The prevalence of such behaviors has led to the rise of an area of research known as \emph{strategic classification}. 
Strategic classification, introduced by~\citet{hardt2016strategic}, aims to understand how a learner can optimally modify machine learning algorithms to be robust to such strategic manipulations of agents, if and when possible.

The strategic classification literature makes the assumption that the model deployed by the learner is \emph{fully observable} by the agents, granting them the ability to optimally best respond to the learner using resources such as effort, time, and money. Yet, this \emph{full information} assumption can be unrealistic in practice. There are several reasons for this: some machine learning models are proprietary and hide the details of the model to avoid leaking ``trade secrets'': e.g., this is the case for the credit scoring algorithms used by FICO, Experian, and Equifax.\footnote{``The exact algorithm used to condense your credit report into a FICO score is a closely guarded secret, but we have a general layout of how your credit score is calculated.'' Source: Business Insider, October 2023. [Link: \url{https://www.businessinsider.com/personal-finance/what-is-fico-score}].} Some classifiers are simply too complex in the first place to be understood and interpreted completely by a human being with limited computational power, such as deep learning models. Other classifiers and models may be obfuscated for data privacy reasons, which are becoming an increasingly major concern with new European consumer protection laws such as GDPR~\citep{GDPR} and with the October 2023 Executive Order on responsible AI~\citep{biden2023executive}.
In turn, there is a need to study strategic classification when agents only have \emph{partial} knowledge of the learner's model. 

There has been a relatively short line of work trying to understand the impact of incomplete information on strategic classification. \citet{altmicro} and \citet{bechavod2021gaming} focus on algorithms for deploying optimal classifiers in settings where agents can only gain partial or noisy information about the model.~\citet{haghtalab2023calibrated} study calibrated Stackelberg games, a more general form of strategic classification; in their framework, the learner engages in repeated interactions with agents who base their actions on calibrated forecasts about the learner's classifier. They characterize the optimal utility of the learner for such games under some regularity assumptions. While we also model agents with a distributional prior knowledge of the learner's actions,~\citet{haghtalab2023calibrated} focus on an online learning setting and the selection of a strategy for the learner without incorporating any form of voluntary \emph{information release} by the learner.

In contrast, we focus on this additional critical aspect of voluntary \emph{information release} by the learner 
that these works do not study. Namely, we ask: how can the learner complement agents' incomplete information by releasing \emph{partial and truthful} information about the classifier; and \emph{how much information} should a learner release to agents to maximize the accuracy of its model? 

This should give the reader pause: why should a learner release information about their deployed classifier since 
presumably such information only makes it easier for agents to manipulate their features and ``trick'' the learner. In fact,~\citet{stratindark} showed that information revelation can help---a learner may prefer to fully reveal their classifier as opposed to hiding it. While they consider either fully revealing the classifier or completely hiding it, our model considers a wider spectrum of information revelation that includes both ``full-information-release'' and ``no-information-release''. We show that there exist instances where it is optimal to reveal only \emph{partial} information about the classifier, in a model where a learner is \emph{allowed to reveal a subset of the classifiers containing the true deployed classifier.} For example, a tech firm might reveal to candidates that they will ask them about lists and binary trees during their job interviews. Lenders might reveal to clients that they do not consider factors like collateral. In the following, we summarize our contributions.

\paragraph{\textbf{Summary of contributions:}} 
\begin{itemize}
\item We propose a new model of interactions between strategic agents and a learner, under partial information. In particular, we introduce two novel modeling elements compared to the standard strategic classification literature: i) agents have partial knowledge about the learner's classifier in the form of a \emph{distributional prior} over the hypothesis class, and ii) the learner can \emph{release partial information} about their deployed classifier. Specifically, our model allows the learner to release a \emph{subset} of the hypothesis class to narrow down the agents' priors. 
Given our model, we consider a (Stackelberg) game between agents with partial knowledge and a learner that can release partial information about its deployed model. On the one hand, the agents aim to manipulate their features, at a cost, to increase their likelihood of receiving a positive classification outcome. On the other hand, the learner can release partial information to maximize the expected accuracy of its model, after agent manipulations. 
\item We study the agent's best response in our game.
We show that while in general, it is intractable to compute the best response of the agents in our model, there exist oracle-efficient algorithms\footnote{An oracle-efficient algorithm is one that calls a given oracle only polynomially many times.} that can \emph{exactly} solve the best response of the agents when the hypothesis class is the class of low-dimensional \emph{linear} classifiers.
\item We then move away from the linearity assumption on the hypothesis class and consider a natural condition on the agents' cost function for which we give an oracle-efficient \emph{approximation} algorithm for the best response of the agents for \emph{any} hypothesis class.
\item Finally, we study the learner's optimal information release problem. We consider screening settings where agents are represented to the learner by a real-valued number that measures their qualification level for a certain task. Prior work has focused on similar one-dimensional settings in the context of strategic classification; see, e.g., \citep{beyhaghi_et_al, randnoisestrat}. We first show that the learner's optimal information release problem is NP-hard when the agents' prior can be arbitrary.
\item In response to this hardness result, we focus on \emph{uniform} prior distributions and provide closed-form solutions for the case of \emph{continuous} uniform priors, and an efficient algorithm to compute the optimal information release for the case of \emph{discrete} uniform priors.
\item We finally consider alternative utility functions that are based on false positive (or negative) rates for the learner and provide insights as to what optimal information release should look under these utility functions, without restricting ourselves to uniform priors.
\end{itemize}

\paragraph{\textbf{Related Work.}}

Strategic classification  was first formalized by \citet{bruckner2011stackelberg, hardt2016strategic}.
\citet{hardt2016strategic} is perhaps the most seminal work in the area of strategic classification: they provide the first computationally efficient algorithms (under assumptions on the agents' cost function) to efficiently learn a near-optimal classifier in strategic settings. Importantly, this work makes the assumption that the agents fully know the exact parameters of the classifier due to existing ``information leakage'', even when the firm is obscuring their model.~\citet{hardt2016strategic} also do not consider a learner that can release partial information about their model. 

There have been many follow-up works in the space of strategic classification, including~\citep{kleinberg2020classifiers, randnoisestrat, miller2020strategic,strat1,strat3,strat4,strat5,strat6,strat8,strat9,strat10,strat11,strat12,strat13,strat14,strat15,strat16,stratperceptron,tang2021linear,hu2019disparate,socialcost18,performative2020,gameimprove,bechavod2021gaming,shavit2020causal,dong2018strategic,chen2020learning,harris2021stateful,shao2023strategic,ahmadi2023fundamental,zhang2021incentive,cohen2023sequential}. Each of these papers augments the original strategic classification work with additional and relevant considerations such as fairness, randomization, and repeated interaction.

Closest to our work,~\citet{altmicro},~\citet{stratindark},~\citet{bechavod2022information}, and~\citet{haghtalab2023calibrated} relax the full information assumption and characterize the impact of opacity on the utility of the learner and agents.~\citet{altmicro} are the first to introduce a model of ``biased'' information about the learner's classifier: instead of observing the learner's deployed classifier exactly, agents observe and best respond to a noisy version of this classifier; one that is randomly shifted (by an additive amount) from the true deployed classifier.
In contrast,~\citet{stratindark} and~\citet{bechavod2022information} consider models of \emph{partial} information on the classifiers, where agents can access samples in the form of historical (feature vector, learner's prediction) pairs. More precisely,~\citet{stratindark} study what they coin the ``price of opacity'' in strategic classification, defined as the difference in prediction error when not releasing vs fully releasing the classifier. They are the first to show that this price can be positive (in the context of strategic classification), meaning that a learner can reduce their prediction error by fully releasing their classifier in strategic settings.~\citet{bechavod2022information} consider a strategic regression setting in which the learner does not release their regression rule, but agents have access to (feature, score) samples as described above. They study how disparity in sample access (e.g., agents may only access samples from people similar to them) about the classifier across different groups induce unfairness in classification outcomes across these groups.~\citet{haghtalab2023calibrated} consider agents with (calibrated) forecasts over the actions of the learner, but do not consider the learner's information release which is our focus. Additionally, in our model, we do not constrain the agent's prior distribution to be calibrated. 

Beyond strategic classification, there are a few related lines of work where such partial information is considered. One is Bayesian Persuasion \citep{kamenica2011bayesian}; our work can be seen as similar to Bayesian Persuasion, but in a slightly different setting where i) the state of the world is a binary classifier and ii) the sender's signaling scheme \emph{must} incorporate an additional truthfulness constraint: namely, the sender's signal is restricted such that it \emph{cannot} rule out the true classifier (or, in Bayesian Persuasion terms, the state of the world). Another is algorithmic recourse; e.g.,~\citet{harris2022bayesian} focus on a similar ``intermediate'' information release problem as ours, but their approach relies on publishing a recommended (incentive-compatible) action or recourse for each agent to take, rather than a set of potential classifiers used by the learner. 
In our model, we release the same signal or information to all agents based on the underlying distribution over these agents' features.

\section{Model}\label{sec:model}
Our model consists of a population of \emph{agents} and a \emph{learner}. Each agent in our model is represented by a pair $(x,y)$ where $x \in \mathcal{X}$ is a feature vector, and $y \in \{0,1\}$ is a binary label. Throughout, we call an agent with $y=0$ a ``negative'', and an agent with $y=1$ a ``positive''. We assume there exists a mapping $f: \mathcal{X} \to \{0,1\}$ that governs the relationship between $x$ and $y$; i.e., $y=f(x)$ for every agent $(x,y)$. We will therefore use $x$ to denote agents from now on. 
We denote by $D$ the distribution over the space of agents $\mathcal{X}$. Similar to standard strategic settings, agent manipulations are characterized by a cost function $c: \mathcal{X} \times \mathcal{X} \to [0, \infty)$ where $c(x,x')$ denotes the cost that an agent incurs when changing their features from $x$ to $x'$. Let $\mathcal{H} \subseteq \{0,1\}^\mathcal{X}$ denote our hypothesis class, and let $h \in \mathcal{H}$ be the classifier that the learner is using for classification.

\paragraph{\textbf{A Partial Knowledge Model for the Agents.}} 
 We move away from the standard assumption in strategic classification that agents fully know $h$ and model agents as having a \emph{common} (shared by all agents) \emph{prior distribution} $\pi$ over $\mathcal{H}$. This distribution captures their \emph{initial} belief about which classifier is deployed by the learner. More formally, for every $h' \in \mathcal{H}$, $\pi(h')$ is the probability that the learner is going to deploy $h'$ for classification \emph{from the agents' perspective.} We emphasize that the learner is committed to using a fixed classifier $h$---$\pi$ only captures the agents' belief about the deployed classifier.

Throughout the paper, we assume that the deployed classifier $h$ is in the support of $\pi$. This assumption is consistent with how an agent would form their prior in real life: the learner (e.g., a hiring company) has been using a classifier $h$ (the hiring algorithm) to screen agents (applicants). The use of $h$ in the past informs the agents (future applicants) and shapes their prior belief $\pi$ about the classifier. Therefore, $h$ belonging to the support of $\pi$ is only a mild assumption, given that repeated use of $h$ in the past has led to the formation of $\pi$. 

\paragraph{\textbf{A Partial Information Release Model for the Learner.}} The learner has the ability to influence the agents' prior belief $\pi$ about the deployed classifier $h$ by releasing partial information about $h$. We model information release by releasing a subset $H \subseteq \mathcal{H}$ such that $h \in H$. We note that we reveal information truthfully, meaning that the deployed classifier is required to be in $H$. Note that this is a general form of information release because it allows the learner to release \emph{any} subset of the hypothesis class, so long as it includes the deployed classifier $h$.
Below, we provide natural examples of information release that can be captured by our model.

\begin{example}[Examples of Information Release via Subsets]
Consider the class of linear halfspaces in $d$ dimensions: $\mathcal{H} = \{ h_{w,b} : w = [w^1, w^2, \ldots, w^d]^\top \in \reals_+^d, \, b \in \reals \}$ where $h_{w,b}(x) \triangleq \mathds{1} [w^\top x + b \ge 0]$ and $x \in \mathcal{X} = \reals^d$ is the feature vector.  Let $h = h_{w_0, b_0}$ be the classifier deployed by the learner for some $w_0, b_0$. Under this setting, revealing the corresponding parameter of a feature, say $x^j$, in $h$ corresponds to releasing
$
H_1 = \{ h_{w,b} \in \mathcal{H}: w^j = w_0^j \}
$.
Let $(i_1, i_2, \ldots, i_k)$ denote the indices of the $k$ largest coordinates of $w_0$ of $h$. Revealing the top $k$ features of the classifier $h$ corresponds to releasing
$
H_2 = \{ h_{w,b} \in \mathcal{H}: w^{i_1}, w^{i_2} \ldots, w^{i_k} \text{ are the $k$ largest coordinates of } w \}
$.
Let $I_0$ be such that $w_0^i \neq 0$ iff $i \in I_0$. Revealing the relevant features of $h$, i.e. features with nonzero coefficients, corresponds to releasing
$
H_3 = \{ h_{w,b} \in \mathcal{H}: w^i \neq 0, \, \forall i \in I_0 \}
$.
Revealing features that the learner is \emph{not} using for classification can be cast similarly using a subset. This is a common form of information release in the real world with some companies publishing a disclaimer that they do not use, e.g., sensitive attributes like race or gender in their decisions.
\end{example}

\paragraph{\textbf{The Strategic Game with Partial Information Release.}} Once the partial information $H$ is released by the learner, agents best respond as follows: each agent first computes their \emph{posterior} belief about the deployed classifier 
by projecting their prior $\pi$ onto $H$, which we denote by $\pi \vert_H$, and is formally defined by
\[
\forall h' \in \mathcal{H}, \ \pi \vert_H (h') \triangleq \frac{\pi (h')}{\pi (H)} \mathds{1} [h' \in H] 
\]
Given this posterior distribution, the agent then moves to a new point that maximizes their utility. The utility is \emph{quasi-linear} and measured by the probability (according to $ \pi \vert_H$) of receiving a positive outcome minus the manipulation cost. Formally, the utility of agent $x$ that manipulates to $x'$, under the partial information $H$ released by the learner is given by
\begin{equation}
     u_x (x',H) \triangleq \Pr_{h' \sim \pi  \vert_H} \left[h'(x') = 1\right] - c(x, x')
\end{equation}
We let $\br (x,H)$ denote the best response of agent $x$, 
i.e. a point $x'$ that maximizes $u_x (x',H)$. \footnote{When there are several utility-maximizing solutions, we always break ties in favor of the lowest cost solution.}
The goal of the learner is to release $H$ that includes its deployed classifier $h$ so as to maximize its utility which is measured by its expected strategic accuracy.
\begin{equation}\label{eq:utilityPartial}
    U (H) \triangleq \Pr_{x \sim D} \left[ h(\br(x,H)) = f(x)\right]
\end{equation}

\begin{definition}[Strategic Game with Partial Information Release] 
\label{defn:partialgame}The game, between the learner and the agents, proceeds as follows:
    \begin{enumerate}
        \item The learner (knowing $f$, $D$, $c$, $\pi$) publishes a subset of hypotheses $H \subseteq \mathcal{H}$ such that $h \in H$.
        \item Every agent $x$ best responds by moving to a point $\br(x, H)$ that maximizes their utility.
        \begin{equation*}
         \br(x, H) \in \argmax_{x' \in \mathcal{X}} u_x (x',H)
        \end{equation*}
        \end{enumerate}
        The learner's goal is to find a subset $H^\star \subseteq \mathcal{H}$ with $h \in H^\star$, that maximizes its utility\footnote{We emphasize that here $h$ is fixed -- namely, $H$ is the only variable in the optimization problem of the learner which is constrained to include $h$.}:
        \begin{equation*}
        H^\star \in \argmax_{H \subseteq \mathcal{H}, \, h \in H} U (H)
        \end{equation*}
\end{definition}

We note that similar to standard strategic classification, the game defined in Definition~\ref{defn:partialgame} can be seen as a \emph{Stackelberg} game in which the learner, as the ``leader'', commits to her strategy first and then the agents, as the ``followers'', respond. The optimal strategy of the learner, $H^\star$, corresponds to the \emph{Stackelberg equilibrium} of the game, assuming best response of the agents.

\paragraph{\textbf{Contrasting with the Standard Setting of Strategic Classification.}} 
The game we define in Definition~\ref{defn:partialgame} not only captures both the partial knowledge of the agents and the partial information release by the learner but can also be viewed as a \emph{generalization} of the standard strategic classification game where the agents fully observe the classifier $h$, which we refer to as the \emph{full information release} game (e.g., see \citep{hardt2016strategic}). This is because the learner can choose $H = \{ h \}$ in or model. 
But can partial information release lead to significant improvement in the utility of the learner, compared to full information release?

Observe that by definition, $U(H^\star) \ge U ( \{ h \})$, i.e., the learner can only gain 
utility when they optimally release partial information instead of fully revealing the classifier. In the following examples, we show that there exist instantiations of the problem where $U (H^\star) > U ( \{ h\})$, even when $h$ is picked to be the optimal classifier in the full information release game, i.e., one that maximizes $U ( \{ h\})$. In other words, we show that the learner can gain \emph{nonzero} utility by releasing a subset that is not $\{ h \}$, \emph{even 
if the choice of $h$ is optimized for the full information release game}. 

\begin{example}[Partial vs. Full Information Release]\label{ex:1}
    Suppose $\mathcal{X} = \{x_1, x_2\}$, and that their probability weights under the distribution\footnote{The claim holds for any distribution $D$ in which both $x_1$ and $x_2$ are in the support.} are given by
    $
    D(x_1) = 2/3, D(x_2) = 1/3
    $,
    and their true labels are given by $f(x_1) = 1$, $f(x_2) = 0$.    
    Suppose the cost function is given as follows:
    $
    c(x_1, x_2) = 2, c(x_2, x_1)= 3/4
    $.
    Let $\mathcal{H} = \{ h_1, h_2, h_3\}$ be given by table~\ref{tab:hypothesis}.
    \begin{table}[t]
        \centering
        \begin{tabular}{c|c|c|c}
              & $h_1 (=f)$ & $h_2$ & $h_3$ \\
              \hline
             $x_1$ & 1 & 0 & 0 \\
             $x_2$ & 0 & 1 & 0
        \end{tabular}
        \caption{Hypothesis class $\mathcal{H}$ in Example~\ref{ex:1}}
        \label{tab:hypothesis}
    \end{table}
    One can show that under this setting, $h = h_1$ is the optimal classifier under full information release, i.e., it optimizes $U (\{h\})$, and that for such $h$, $U ( \{ h \}) = 2/3$. However, suppose the prior distribution over $\mathcal{H}$ is uniform. 
    One can show that under this setting, and when $h =  h_1$ is the deployed classifier, releasing $H^\star = \{h_1, h_2\}$ implies $U (H^\star) = 1 > U ( \{ h \}) = 2/3$. In other words, the learner can exploit the agent's prior by releasing information in a way that increases its own utility by a significant amount.
\end{example}

In the next example, we consider the more natural setting of single-sided threshold functions in one dimension and show that the same phenomenon occurs: the optimal utility achieved by partial information release is strictly larger than the utility achieved by the full information release of $h$, \emph{even after the choice of $h$ is optimized for full information release}.

\begin{example}[Partial vs.~Full Information Release]\label{ex:partailCanbeBetter2}
Suppose $\mathcal{X} = [0,2]$, $D$ is the uniform distribution over $[0,2]$, $f(x) = \mathds{1} \left[ x \ge 1.9 \right]$, $\mathcal{H} = \{ h_t : t \in [0,2] \}$ where $h_t(x) \triangleq \mathds{1} \left[ x \ge t \right]$. Suppose the cost function is given by the distance $c(x,x') = |x - x'|$. We have that under this setting, the optimal classifier in $\mathcal{H}$ under full information release is $h = h_2$, and that its corresponding utility is
$
U (\{ h \}) = 1 - \Pr_{x \sim Unif[0,2]} \left[ 1 \le x < 1.9 \right] = 0.55
$.
Now suppose the agents have the following prior over $\mathcal{H}$:
$
\pi (h') = 0.1 \cdot \mathds{1} [h' = h_2] + 0.9 \cdot \mathds{1} [h' = h_{1.8}]
$.
Under this setting, and when $h= h_2$ is deployed for classification, one can see that releasing $H^\star = \{ h_2, h_{1.8}\}$ leads to perfect utility for the learner. We therefore have
$
U (H^\star) = 1 > U (\{ h \}) = 0.55.
$.
\end{example}


\section{The Agents' Best Response Problem}
\label{sec:agentbr}
In this section we focus on the best response problem faced by the agents in our model, as described in Definition~\ref{defn:partialgame}. In particular, we consider a natural optimization oracle for the cost function of the agents that can solve simple projections. We will formally define this oracle later on. Given access to such an oracle, we then study the \emph{oracle complexity}\footnote{The number of times the oracle is called by an algorithm.} of the agent's best response problem. First, we show in subsection~\ref{subsec:br-hardness} that the best response problem is computationally hard even when we restrict ourselves to a common family of cost functions. In subsection~\ref{subsec:linear}, we provide an \emph{oracle-efficient} algorithm\footnote{An algorithm that calls the oracle only polynomially many times.} for solving the best response problem when the hypothesis class is the class of low-dimensional linear classifiers. In subsection~\ref{subsec:submodular}, we introduce the notion of \emph{submodular cost functions} and show that for any hypothesis class, there exists an oracle-efficient algorithm that \emph{approximately} solves the best response problem when the cost function is submodular.

Recall that the agents' best response problem can be cast as the following: given an agent $x \in \mathcal{X}$, and a distribution $P$ (e.g., $P=\pi \vert_H$ where $\pi$ is the prior and $H$ is the released information) over a set $\{h_1, \ldots, h_n\} \subseteq \mathcal{H}$, we want to solve
\[
\argmax_{z \in \mathcal{X}} \left\{ \Pr_{h' \sim P} \left[h' (z) = 1\right] - c(x, z) \right\}
\]
We consider an oracle that given any region $R \subseteq \mathcal{X}$, which is specified by the intersection of positive (or negative) regions of $h_i$'s, returns the projection of the agent $x$ onto $R$ according to the cost function $c$: $\argmin_{z \in R} c(x,z)$. For example, when $\mathcal{H}$ is the class of linear classifiers and $c (x,z) = \Vert x - z \Vert_2$, the oracle can compute the $\ell_2$-projection of the agent $x$ onto the intersection of any subset of the linear halfspaces in $\{h_1, \ldots, h_n\}$. We denote this oracle by $\text{Oracle} (c, \mathcal{H})$ and formally define it in Algorithm~\ref{alg:oracle}.

\begin{algorithm}[t]
	\SetAlgoNoLine
  	 \KwIn{ agent $x$, region $R = R^+ \cap R^-$ specified as,
    	$
    	R^+ = \cap_{ \ i \in I^+} \left\{ z: h_i (z) = 1 \right\}$ and  $R^- = \cap_{ \ i \in I^-} \left\{ z: h_i (z) = 0 \right\}$ for some $I^+$ and $I^{-}$.}
    	\KwOut{
    	$
    	\argmin_{z \in R} c(x,z)
    	$}
	\caption{Oracle($c, \mathcal{H}$)}
	\label{alg:oracle}
\end{algorithm}

Having access to such an oracle, and without further assumptions, the best response problem can be solved by exhaustively searching over all subsets of $\{h_1, \ldots, h_n\}$ because:
\begin{equation}\label{eq:br}
    \max_{z \in \mathcal{X}} \left\{ \Pr_{h' \sim P} \left[h'(z) = 1\right] - c(x, z) \right\} = \max_{S \subseteq \{h_1, \ldots, h_n\}} \left\{ \sum_{h' \in S } P (h') - \min_{z: h'(z) = 1, \, \forall h' \in S} c(x,z) \right\}
\end{equation}
This algorithm is inefficient because it makes exponentially ($2^n$) many oracle calls. In what follows, we consider natural instantiations of our model and examine if we can get algorithms that make only $poly(n)$ oracle calls.

\subsection{Computational Hardness for $p$-Norm Cost Functions}\label{subsec:br-hardness}
In this section, we consider Euclidean spaces and the common family of $p$-norm functions for $p\ge 1$ and show that even under the assumption that the cost function of the agent belongs to this family, the problem of finding the best response is computationally hard.
Formally, a $p$-norm cost function is defined by: for every $x,x'\in \reals^d$, $c_p(x, x') = \|x - x'\|_p$ where $p\ge 1$.

\begin{theorem}[Computational Hardness with Oracle Access]\label{thm:bestresponsehardness}
     $\Omega(2^n / \sqrt{n})$ calls to the oracle (Algorithm~\ref{alg:oracle}) are required to compute the best response of an agent with a $2/3$ probability of success, even when $\mathcal{X} = \mathbb{R}^2$ and the cost function is $c_p$ for some $p\ge 1$.
\end{theorem}

\begin{proof}[Proof of Theorem~\ref{thm:bestresponsehardness}]
To prove the claim, we reduce the following hidden-set detection problem with \textsc{EqualTo($\cdot$)} oracle to our best response problem. In hidden-set detection problem, given two players, Alice and Bob, with Bob possessing a `hidden' subset $S^\star \subseteq [n]$ of size $n/2$, Alice aims to identify Bob's set $S^\star$ using the minimum number of queries to Bob. She has query access to \textsc{EqualTo($T$)} oracle that checks whether her set $T \subset [n]$ matches Bob's set $(S^\star)$. It is trivial that any randomized algorithm for the hidden-set detection problem with success probability of at least $1 - O(1)$ requires ${n \choose n/2}$ queries in the worst-case scenario. This is via a straightforward application of Yao's Min-Max principle~\cite{yao1977probabilistic}: consider a uniform distribution over all subsets of size $n/2$ from $[n]$, as the Bob's set. Then, after querying half of the subsets of size $n/2$, the failure probability of Alice in detecting Bob's set is at least $(1-1/n)(1-1/(n-1))\cdots (1 - 1/(n/2)) > (1 - 2/n)^{n/2} >  e^{-1}(1-2/n) > 1/3$ for sufficiently large values of $n$.

Next, corresponding to an instance of the hidden-set detection problem with $S^\star$, we create an instance of the agents' best response problem and show that any algorithm that computes the best response with success probability at least $2/3$ using $N$ oracle calls (Algorithm~\ref{alg:oracle}), detects the hidden set of the given instance of the hidden-set detection problem using at most $N$ calls of \textsc{EqualTo($\cdot$)} with probability at least $2/3$. Hence, computing the best response problem with success probability at least $2/3$ requires ${n \choose n/2} = \Omega(2^{n}/\sqrt{n})$ oracle calls.

Let $n = 2k$ and $\epsilon < 1/n$. Corresponding to every subset $S\subset [n]$ of size $n/2 - 1$, there is a distinct point $x_S$ at distance $1/2 - \epsilon$ from the origin, i.e.,  $\|x_S\|_p = 1/2 -\epsilon$. Corresponding to every subset $S\subset [n]$ of size $n/2$, there are two distinct points $x_{S,n}$ and $x_{S,f}$ at distances respectively $1/2 -\epsilon$ (near) and $1/2 + \epsilon$ (far) from the origin, i.e.,  $\|x_{S,n}\|_p = 1/2 - \epsilon$ and $\|x_{S,f}\|_p = 1/2 + \epsilon$.  

Now, we are ready to describe the instance $I_{S^\star}$ of our best response problem corresponding to the given hidden-set detection problem with $S^\star$.  
We define $\mathcal{H} = \{h_1, \cdots, h_n\}$ and distribution $P$ over $\mathcal{H}$ such that,
\begin{itemize}
    \item $P$ is a uniform distribution over all classifiers $\mathcal{H}$, i.e.,  $P(h_i) = 1/n$ for every $i\in [n]$. 
    \item For every subset $T \subset [n]$ of size $n/2 -1$, 
    we define $h_i(x_T) = \mathds{1}[i \in T]$.

    \item For every subset $T \subset [n]$ of size $n/2$, we define $h_i(x_{T,f}) = \mathds{1}[i \in T]$. Moreover, if $T\neq S^\star$, then $h_i(x_{T,n}) = 0$ for all $i\in [n]$. Otherwise, if $T = S^\star$, 
    we define $h_i(x_{T,n}) = \mathds{1}[i \in T]$.

    \item Finally, for the remaining points in $\mathcal{X}$, i.e., $x' \in \mathbb{R}^2\setminus(\{x_T: T\subset[n] \text{ and } |T| = n/2-1\} \cup \{x_{T,n}, x_{T, f}: T\subset[n] \text{ and } |T| = n/2\})$, we define $h_i(x') =0$ for all $i\in [n]$. In other words, points that do not correspond to subsets of size $n/2-1$ or $n/2$ are classified as negative examples by every classifier in $\mathcal{H}$.
\end{itemize}
In the constructed instance $I_{S^\star}$, no point is classified as positive by more than $n/2$ classifiers in $\mathcal{H}$, and the $p$-norm distance from the origin for all points classified as positive by a subset of classifiers is at least $1/2-\epsilon$. Therefore, the best response for an agent located at the origin of the space is $x_{S^\star,n}$, yielding a utility of $1/2 - (1/2 - \epsilon) = \epsilon > 0$. Hence, the computational task in computing the best response involves identifying the (hidden) subset $S^\star$. Refer to Figure~\ref{fig:instance} for a description of $I_{S^*}$.

\begin{figure}
    \centering
    \begin{tikzpicture}
        \node[circle,fill,inner sep=1.5pt] (origin) at (0,0) {};
    
        \draw (origin) circle (2cm);
    
        \foreach \angle in {0,45,...,315} {
            \node[circle,fill=blue,inner sep=1pt] at (\angle:1.9cm) {};
        }
        \foreach \angle in {0,30,...,300} {
            \node[circle,fill=red,inner sep=1pt] at (\angle:2.1cm) {};
        }
        \node[circle,fill=red,inner sep=1.2pt, label=above:{$S^*$}] at (330:1.9cm) {};
    \end{tikzpicture}
    \caption{In this example, we consider $p=2$, i.e., $c(x,x') = \|x, x'\|_2$. The agent is located at the origin. Blue nodes correspond to a point in the intersection of the positive regions of subsets of classifiers of size $\frac{n}{2} - 1$, each located at a Euclidean distance of $1/2 - \epsilon$ from the origin, where $\epsilon$ is a small positive value. Moreover, points in the intersection of the positive regions of subsets classifiers of size $\frac{n}{2}$ are indicated by red points, all except the one corresponding to $S^\star$ are located at a Euclidean distance of $1/2 + \epsilon$ from the origin. The red point corresponding to $S^\star$ is uniquely placed at a distance of $1/2 - \epsilon$ from the origin, similar to the blue nodes. Furthermore, all points, corresponding to different subsets, are located at distinct locations in the space.}
    \label{fig:instance}
\end{figure}
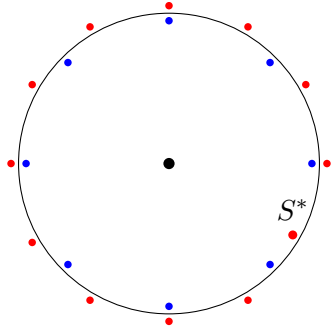

Although we described the construction of $I_{S^\star}$, what we need to show to get the exponential lower bound on the oracle complexity of the best response problem is constructing an oracle (i.e., an implementation of Algorithm~\ref{alg:oracle}), using the $\textsc{EqualTo}(\cdot)$ oracle, consistent with $I_{S^\star}$. To do so, given a subset of classifiers specified by $T \subset [n]$, the oracle returns as follows:
\begin{itemize}
    \item {\bf if $|T| > n/2$}: It returns an empty set.
    \item {\bf if $|T| < n/2$}: It returns $x_{T'}$ for an arbitrary set $T' \supseteq T$ of size $n/2-1$. Note that $\|x_{T'}\|_p = 1/2 - \epsilon$.
    \item {\bf if $|T| = n/2$ and $\textsc{EqualTo} (T) = \textsc{False}$}: It returns $x_{T, f}$. Note that $\|x_{T,f}\|_p = 1/2 + \epsilon$.
    \item {\bf if $|T| = n/2$ and $\textsc{EqualTo} (T) = \textsc{True}$}: It returns $x_{T, n}$. Note that $\|x_{T,n}\|_p = 1/2 - \epsilon$.   
\end{itemize}
\end{proof}
\begin{remark}
As in each instance $I_{S^\star}$ the only point with strictly positive utility is $x_{S^*,n}$, our proof for Theorem~\ref{thm:bestresponsehardness} essentially rules out the existence of any approximation algorithm for the best response problem with success probability at least $2/3$ using $o(2^n/\sqrt{n})$. 
\end{remark}

\subsection{An Oracle-Efficient Algorithm for Low-Dimensional Linear Classifiers}\label{subsec:linear}

In this section we show that when $\mathcal{X} = \reals^d$ for some $d$,
and when $\mathcal{H}$ contains only linear classifiers, i.e., every $h \in \mathcal{H}$ can be written as $h(x) = \mathds{1} \left[ w^\top x + b \ge 0 \right]$ for some $w \in  \reals^d$ and $b \in \reals$, then the best response of the agents can be computed with $O(n^d)$ oracle calls when $d \ll n$. 

The algorithm, which is described in Algorithm~\ref{alg:BR_linear}, first computes the partitioning ($R_n$) of the space $\mathcal{X}$ given by the $n$ linear classifiers. For any element of the partition in $R_n$, it then solves the best response when we restrict the space to that particular element. This gives us a set of points, one for each element of the partition. The algorithm finally outputs the point that has maximum utility for the agent. This point, by construction of the algorithm, is the best response of the agent. The oracle-efficiency of the algorithm follows from the observation that $n$ linear halfspaces in $d$ dimensions partition the space into at most $O (n^d)$ regions when $d \ll n$. The formal guarantee of the algorithm is given in Theorem~\ref{thm:linear}.

\begin{algorithm}[t]
\SetAlgoNoLine
    \KwIn{agent $x$, cost function $c$, arbitrary distribution $P$ over linear classifiers $\{h_1, \ldots, h_n\}$}
    \textbf{Step 1. Compute the partitioning ($R_n$) of the space given by $\{h_1, \ldots, h_n\}$}\;
    Initialize $R_1 \gets \left\{ \{z: h_1 (z) = 1 \}, \{z: h_1 (z) = 0 \} \right\}$\;
    \For{$i=2, \ldots, n$}{
        $R_i \gets R_{i-1}$\;
        \For{$R \in R_{i-1}$}{
            \If{$\{z: h_{i}(z) = 0 \} \cap R \neq \emptyset$ }{
                $R_i \gets R_i \setminus R$ \tcp*{Remove $R$}
                $R_i \gets R_i \cup \left\{ \{ z: h_i (z) = 1 \} \cap R, \{ z: h_i (z) = 0 \} \cap R \right\}$ \tcp*{Split $R$ into two regions}
            }
        }
    }
    \textbf{Step 2. Given $R_n$, compute the best response}\;
    \For{$R \in R_n$}{
        Let $R = R^{+} \cap R^{-}$ where
        $
        R^+ = \cap_{ \ i \in I^+} \left\{ z: h_i (z) = 1 \right\}$ and $R^- = \cap_{ \ i \in I^-} \left\{ z: h_i (z) = 0 \right\}
        $\;
        Call the oracle (Algorithm~\ref{alg:oracle}) to solve
        $
        z _R \in \argmin_{z \in R}  c(x, z) 
        $\;
        Compute the utility of $z_R$:
        $
        \text{utility} (z_R) = \sum_{i \in I^+} P (h_i) - c(x,z_R)
        $\;
    }
    \KwOut{$\argmax_{z \in Z} \text{utility} (z)$ where $Z = \{ z_R: R \in R_n \}$}
\caption{Best Response of Agents in the Linear Case}
\label{alg:BR_linear}
\end{algorithm}

\begin{theorem}\label{thm:linear}
    Suppose $\mathcal{X} = \reals^d$ for some $d \ll n $, and $\mathcal{H}$ contains only linear classifiers. Then for any agent $x$, any cost function $c$, and any distribution $P$ over $\{h_1, \ldots, h_n\} \subseteq \mathcal{H}$, Algorithm~\ref{alg:BR_linear} returns the best response of the agent in time $O(n^{d+1})$, while making $O(n^d)$ calls to the oracle (Algorithm~\ref{alg:oracle}).
\end{theorem}

\begin{proof}[Proof of Theorem~\ref{thm:linear}]
     The fact that Algorithm~\ref{alg:BR_linear} returns the best response follows from the construction of the algorithm. We first prove the oracle and the runtime complexity for $d=2$ and then generalize it to any $d$. The oracle complexity of Algorithm~\ref{alg:BR_linear} is $|R_n|$. Note that $|R_1| = 2$, and for any $n \ge 2$, $|R_n| \le |R_{n-1}| + n$.   
    This is because the line $\{z: h_n(z) = 0\}$ intersects the lines formed by $\{h_1, \ldots, h_{n-1}\}$ in at most $n-1$ points, which will then partition $\{z: h_n(z) = 0\}$ into at most $n$ segments. Each segment of the new line then splits a region in $R_{n-1}$ into two regions. So, there are at most $n$ new regions when $h_n$ is introduced. The recursive relationship implies that
    $
    |R_n| \le 1 + \frac{n(n+1)}{2} = O(n^2)
    $.
    The runtime complexity of the algorithm is then given by
    $
    O\left(\sum_{i=1}^n |R_i| \right) = O(n^3)
    $.
    
    Now consider any dimension $d$ and let $R(n,d)$ denote the number of partitions induced by the classifiers $\{h_1, \ldots, h_n\}$. Note that in this case we have $R(n,d) \le R(n-1,d) + R(n-1,d-1)$.
    The first term on the right hand side is the number of regions induced by $\{h_1, \ldots, h_{n-1}\}$ and the second term is the number of splits (dividing a region into two) when $h_n$ is introduced. Note that $\{z: h_n(z) = 0\}$ is a $d-1$-dimensional hyperplane and the number of splits induced by $h_n$ is simply the number of regions induced by $\{h_1, \ldots, h_{n-1}\}$ on $\{z: h_n(z) = 0\}$, which is $R(n-1,d-1)$.
    The recursive relationship implies that
    $
    |R_n| = R(n,d) \le \sum_{j=0}^d {n \choose j} = O (n^d)
    $.
\end{proof}

\subsection{An Oracle-Efficient Approximation Algorithm for $V$-Submodular Costs}\label{subsec:submodular}

In this section we give a sufficient condition on the cost function under which we can give an approximation algorithm for the best response of the agents. In particular, for a given collection of classifiers $V \subseteq \mathcal H$, we introduce the notion of \emph{$V$-submodular} cost functions which is a natural condition that can arise in many applications. Borrowing results from the literature on submodular optimization, we then show that for any distribution $P$ over a set $V=\{h_1, \ldots, h_n\}$, if the cost function is $V$-submodular, there exists an oracle-efficient approximation algorithm for the best response problem. Recall, from Equation~\eqref{eq:br}, that the best response problem faced by agent $x$ can be written as
\[
\max_{S \subseteq \{h_1, \ldots, h_n\}} g_x (S) \triangleq  \sum_{h' \in S } P (h') - c(x, S)
\]
where, with slight abuse of notation, we define
\begin{equation}\label{eq:inducedcost}
c(x, S) \triangleq \min_{z: h'(z) = 1, \, \forall h' \in S} c(x,z)
\end{equation}
For any $S \subseteq \{h_1, \ldots, h_n\}$, $c(x,S)$ is simply the minimum cost that the agent $x$ has to incur in order to pass all classifiers in $S$, and can be computed via the oracle (Algorithm~\ref{alg:oracle}). We now state our main assumption on the cost function:

\begin{definition}[$V$-Submodularity]\label{def:submodular}
    Let $V = \{h_1, \ldots, h_n\}$ be any collection of classifiers. We say a cost function $c$ is $V$-submodular, if for all $x$, the set function $c(x,\cdot):2^V \to \reals$ defined in Equation~\ref{eq:inducedcost} is submodular: for every $S,S' \subseteq V$ such that $S \subseteq S'$ and every $h' \notin S'$,
    \[
    c \left(x, S\cup\{h'\} \right) - c \left(x, S \right) \ge c\left(x, S'\cup\{h'\} \right) - c \left(x, S' \right)
    \]
\end{definition}

This condition asks that the marginal cost of passing the new classifier $h'$ is smaller when the new classifier is added to $S'$ versus $S$, for any such $h', S, S'$.

Fix a collection of classifiers $V$. Informally speaking, a cost function is $V$-submodular if the agent's manipulation to pass a classifier only helps her (i.e., reduces her cost) to pass other classifiers: the more classifiers the agent passes, it becomes only easier for her to pass an additional classifier. This can happen in settings where some of the knowledge to pass a certain number of tests is \emph{transferable} across tests. Some real-life examples include: 1) a student that is preparing for a series of math tests on topics like probability, statistics, and combinatorics. 2)  a job applicant who is applying for multiple jobs within the same field and preparing for their interviews.

We give a formal example of a $V$-submodular cost function below. In particular, we show that when $\mathcal{X} = \reals$, the cost function $c(x,x') = |x - x'|$ is $V$-submodular where $V$ can be any set of single-sided threshold classifiers. We provide the proof of the claim in Appendix~\ref{app:c}.

\begin{restatable}{claim}{submodular}\label{clm:submodular}
Let $\mathcal{X} = \reals$, and $V = \{h_1, \ldots, h_n\}$ where every $h_i$ can be written as $h_i (x) = \mathds{1} [ x  \ge t_i]$ for some $t_i \in \reals$. We have that the cost function $c(x,x') = |x - x'|$ is $V$-submodular.
\end{restatable}


We now state the main result of this section:

\begin{theorem}\label{thm:submodular}
    Fix any $\mathcal{H}$ and any distribution $P$ over some $V = \{h_1, \ldots, h_n\} \subseteq \mathcal{H}$. If the cost function $c$ is $V$-submodular, then there exists an algorithm that for every agent $x$ and every $\epsilon > 0$, makes $\Tilde{O}(n / \epsilon^2)$ calls to the oracle (Algorithm~\ref{alg:oracle}) and outputs a set $\hat{S} \subseteq V$ such that
    $
    g_x (\hat{S}) \ge \max_{S \subseteq V} g_x (S) - \epsilon
    $.
\end{theorem}

\begin{proof}[Proof of Theorem~\ref{thm:submodular}]
    Note that when the cost function is $V$-submodular, the objective function $g_x$ can be written as the difference of a monotone non-negative modular function\footnote{A set function $r$ is modular if $r(S) = \sum_{s \in S} r(s)$ for any $S$.} and a monotone non-negative submodular function:
    $
    g_x : 2^V \to \reals, \ g_x (S) =  \sum_{h' \in S } P (h') - c(x, S)
    $.
    The result then follows from \citep{el2020optimal} where they provide an efficient algorithm for approximately maximizing set functions with such structure.
\end{proof}

\section{The Learner's Optimization Problem}\label{sec:defense}

In this section we focus on the learner's optimization problem as described in Definition~\ref{defn:partialgame}. The learner is facing a population of agents with prior $\pi$ and wants to release partial information $H\subseteq \mathcal H$ so as to maximize its utility $U(H)$. We note that the learner's strategy space can be restricted to the support of the agents' prior $\pi$ because including anything in $H$ that is outside of $\pi$'s support does not impact $U(H)$. Therefore, one naive algorithm to compute the utility maximizing solution for the learner is to evaluate the utility on all subsets $H \subseteq \text{support} (\pi)$ and output the one that maximizes the utility. However, this solution is inefficient; instead, can we have computationally efficient algorithms? We provide both positive and negative results for a natural instantiation of our model which is introduced below.

\subsection{The Setup: Classification Based on Test Scores}\label{subsec:setup}
Motivated by screening problems such as school admissions, examinations, and hiring, where an individual's qualification level can be captured via a real-valued number, say, a test score, we consider agents that live in the one dimensional Euclidean space: $\mathcal{X} = [0,B] \subseteq \reals$ for some $B$. One can think of each $x$ as the corresponding qualification level or test score of an agent where larger values of $x$ correspond to higher qualification levels or higher test scores. Because we are in a strategic setting, agents can modify their true feature $x$ and ``trick'' or ``game'' the learner by appearing more qualified than they actually are.

We let $f(x) = \mathds{1} \left[ x \ge t \right]$ for some $t$: there exists some threshold $t$ that separates qualified and unqualified agents. We take the hypothesis class $\mathcal{H}$ to be the class of all single-sided threshold classifiers: every $h' \in \mathcal{H}$ can be written as $h' (x) \triangleq \mathds{1} \left[ x \ge t' \right]$ for some $t'$. We further take the cost function of the agents to be the standard distance metric in $\reals$: $c(x,x') = |x' - x|$.\footnote{Our results can be extended to the case where $c(x,x') = k|x' - x|$ for some constant $k$.}

\begin{remark}
    We emphasize that considering agents in the one-dimensional Euclidean space is only for simplicity of exposition. We basically assume, for an \emph{arbitrary} space of agents $\mathcal{X}$, there exists a function $g: \mathcal{X} \to [0,B]$ such that $f(x) = \mathds{1} [g(x) \ge t]$ for some $t$, and that the cost function is given by $c(x,z) = |g(z) - g(x)|$. Here, $g(x)$ captures the qualification level or test score of an agent $x$. Now observe that we can reduce this setting to the introduced setup of this section: take $\mathcal{X'} = \{ g(x): x \in \mathcal{X} \} \subseteq [0,B]$, $f: \mathcal{X'} \to \{0,1\}$ is given by $f(x') = \mathds{1} [x' \ge t]$, and that the cost function $c: \mathcal{X'} \times \mathcal{X'} \to [0, \infty)$ is given by $c(x',z') = |z' - x'|$.
\end{remark}

\begin{remark}
Note that because every classifier $h' \in \mathcal{H}$ is uniquely specified by a real-valued threshold, for simplicity of our notations, we abuse notation and use $h'$ interchangeably as both a \emph{mapping} (the classifier) and a \emph{real-valued number} (the corresponding threshold) throughout this section. The same abuse of notation applies to $f$ as well.
\end{remark}

The classifier deployed by the learner is some $h \ge f$. We note that it is natural to assume $h \ge f$ because in our setup, higher values of $x$ are considered ``better''. So given the strategic behavior of the agents, the learner only wants to make the classification task ``harder'' compared to the ground truth $f$ --- choosing $h < f$ will only hurt the learner's utility. Because we will extensively make use of the fact that $h \ge f$, we state it as a remark below.

\begin{remark}
    The learner's adopted classifier is some $h \in \mathcal H$ such that $h \ge f$.
\end{remark}

\begin{remark}\label{remark:tie-breaking}
As mentioned in the model section, when there are several utility-maximizing solutions for the agents, we always break ties in favor of the lowest cost solution. Furthermore, each agent $x$ in our setup manipulate \emph{only} to larger values of $x$ ($x' \ge x$); this is formally stated in the first part of Lemma~\ref{lem:facts}. Therefore, the tie-breaking of agents is in favor of smaller values of $x'$ in our setup. In other words, given some released information $H$, an agent $x$ chooses 
\begin{equation}
         \br(x, H) = \min \left\{ \argmax_{x' \ge x} u_x (x',H) \right\} 
\end{equation}
In the rest of this section, when we write $\argmax_{x' \ge x} u_x (x',H)$, we implicitly are taking the smallest $x' \ge x$ that maximizes the utility of the agent $x$.
\end{remark}

We finish this section by stating some useful facts about the agents' best response in our setup. The proof of this lemma is provided in Appendix~\ref{app:a}.

\begin{restatable}{lemma}{facts}\label{lem:facts}
    Fix any prior $\pi$ and any points $x_2 \ge x_1$. We have that, for any $H \subseteq \mathcal H$,
    \begin{enumerate}
        \item $\br(x_1, H) \ge x_1$.
        \item $\br (x_2 , H) \ge \br (x_1 , H)$.
        \item If $\br (x_1 , H) \ge x_2$, then $\br (x_1 , H) = \br (x_2 , H)$.
    \end{enumerate}
\end{restatable}

\subsection{NP-Hardness for Arbitrary Prior Distributions}
In this section, we show that under the introduced setup, the learner's optimization problem is NP-hard if the prior can be chosen arbitrarily. We show this by a reduction from the \emph{subset sum} problem which is known to be NP-hard. 

\begin{theorem}\label{thm:np-hard}
    Consider an arbitrary prior $\pi$ over a set of threshold classifiers $\{h_1, h_2, \ldots, h_n \} \subseteq \mathcal H$ that includes $h$. 
    The problem of finding $H \subseteq \{h_1, h_2, \ldots, h_n \}$ so that $h\in H$ and the learner's utility $U(H)$ is maximized is NP-hard. 
\end{theorem}
\begin{proof}[Proof of Theorem~\ref{thm:np-hard}]
    The proof is via a reduction from the subset sum problem. In particular, we consider a variant of the subset sum problem in which we are given a set of $n$ positive numbers $a_1, \cdots, a_n$, and the goal is to decide whether a subset $S \subset [n]$ such that $\sum_{i\in S} a_i = T := (1/2) \sum_{i\in [n]} a_i$ exists.

    Given an instance of the subset sum problem with input $(\{a_1,\cdots, a_n\}, T:= (1/2) \sum_{i\in [n]} a_i)$, we construct the following instance of our problem with one-dimensional threshold classifiers. Define  $f (x) = \mathds{1} \left[ x \ge 0\right]$, $h (x) = \mathds{1} \left[x \ge 2/3 \right]$, and $h_i (x) = \mathds{1} \left[x \ge 100 + i \right]$ for every $i\in [n]$. Moreover, suppose that the prior distribution of the agents $\pi$ is given by: $\pi(h) = 1/2$ and for every $i\in [n]$, $ \pi (h_i) = a_i / (4T)$. Note that $\pi (h) + \sum_{i}  \pi (h_i) = 1$. Let  the data distribution $D$ be the uniform distribution over $[-1000, 1000]$.

    Intuitively speaking, the inclusion of $h_i$'s in $H$ have no direct effect on the accuracy of the released subset $H$, as they can only lead to a subset of the agents located at $x \ge 100$ to manipulate. However, their presence in $H$ will impact the probability mass of $h$ under the posterior $\pi \vert_H$, which is given by
    $
    \pi \vert_H (h) = \pi(h) / \pi (H) \triangleq \rho_H
    $.
    We will show that the learner can achieve perfect accuracy \emph{if and only if} in the given instance subset sum problem there exists a subset which sums up to $T$. To see this consider the following cases for the released information $H$.
    \begin{itemize}
        \item {\bf Case 1: $\rho_H > 2/3$.}
        For any such $H$, all agents at distance $\rho_H$ from $2/3$ gain positive utility by manipulating to $x'=2/3$. Hence, the utility of such solutions for the learner is given by
        $
        1 - \Pr_{x \sim D}[x\in [\frac{2}{3} - \rho_H,0)] < 1
        $.
        \item {\bf Case 2: $\rho_H < 2/3$.}
        For any such $H$, as all classifiers in $H\setminus \{h\}$ are located at $t > 100$, no agent belonging to $[0,2/3-\rho_H)$ gain positive utility by manipulating to $x'=2/3$. Hence, these points will be misclassified by $h$, and consequently, the utility of such solutions for the learner is given by
        $
        1 - \Pr_{x \sim D}[x\in [0, \frac{2}{3}-\rho_H)] < 1
        $.
        \item {\bf Case 3: $\rho_H = 2/3$.} By similar arguments to the previous cases, all agents belonging to $[0, 2/3)$ manipulate to $x' = 2/3$ and all points with negative labels ($x<0$) stay at their location. Therefore, no one will be misclassified, and therefore, the utility of such solutions is $1$.  
    \end{itemize}

Note that because
$
\rho_H = \frac{\pi(h)}{\pi (H)} = \frac{1/2}{1/2 + \pi (H\cap \{h_1, \cdots, h_n\})}
$,
we have that $\rho_H = 2/3$ if and only if $ \pi (H\cap \{h_1, \cdots, h_n\}) = 1/4$. But
$
\pi (H\cap \{h_1, \cdots, h_n\}) = 1/(4T) \sum_{h_i\in H} a_i
$.
We therefore have that $\rho_H = 2/3$ \emph{if and only if} $\sum_{h_i\in H} a_i = T$. Hence, deciding whether the learner's optimization problem has a solution with perfect utility is equivalent to deciding whether in the given subset sum problem there exists a subset $S \subset [n]$ such that $\sum_{i\in S} a_i = T := (1/2) \sum_{i\in [n]} a_i$.       
\end{proof}

\subsection{A Closed-form Solution for Continuous Uniform Priors}
Given the hardness of the learner's problem for arbitrary prior distributions, we focus on a specific family of priors, namely, uniform priors over a given set, and examine the existence of efficient algorithms for such priors. In this section, we provide closed-form solutions for \emph{continuous} uniform priors. More concretely, we assume in this section that $\pi$ is the uniform distribution over an interval $[a,b] \subset \reals$ that includes $h$. The information release of the learner will then be releasing an interval $H = [c,d] \subseteq [a,b]$ such that $h \in [c,d]$.

\begin{theorem}\label{thm:unif-cont}
    Fix any data distribution $D$ over $\mathcal{X}$. Suppose the prior $\pi$ is uniform over an interval $[a,b]$ for some $a,b$ such that $h \in [a,b]$. Define $H_c \triangleq \left[ c, d \right]$ where $d \triangleq \min \left( b, \max \left( h, f+1 \right) \right)$.
    
    If $b - a < 1$, we have that $H^\star = H_c$ is optimal for any $c \in [a, h]$, with corresponding utility
    \[
    U ( H_c) = \begin{cases}
    1 - \Pr_{x \sim D} \left[ d-1 < x <  f \right] & d-1 < f \\
    1 - \Pr_{x \sim D} \left[ f \le x \le d-1 \right] & d-1 \ge f
    \end{cases}
    \]
    If $b-a \ge 1$, we have that for any $c \in (b-1, h]$, the optimal solution $H^\star$ is given by
    \[ H^\star = 
    \begin{cases}
      H_c &   U ( H_c ) > U ( [a,b] )
      \\
      [a,b] & U ( H_c ) \le U ( [a,b] )
    \end{cases}
    \]
    where $U ( H_c )$ is given above and $U ( [a,b] ) = 1 - \Pr_{x \sim D} \left[ f \le x < h \right]$ is the utility of releasing $[a,b]$.
\end{theorem}

\begin{proof}[Proof of Theorem~\ref{thm:unif-cont}]
    Suppose $H = [c,d] \subseteq [a,b]$ is the released information by the learner. The agents then project their uniform prior $\pi$ over $[a,b]$ onto $H$, which leads to the uniform distribution over $[c,d]$ for $\pi \vert_H$, and then best respond according to $\pi \vert_H$. Therefore, for any agent $x$,
    \[
    \br(x, H) = \argmax_{x' \ge x} \left\{ \Pr_{h' \sim Unif[c,d]} \left[x' \ge h' \right] - ( x' - x) \right\}
    \]
    One can then show that if $d - c \ge 1$, $\br(x, H) = x$ for all $x$ because for any manipulation ($x' > x$), the marginal gain in the probability of receiving a positive classification is less than the marginal cost. Furthermore, if $d - c < 1$, then we have
    \[
    \br(x, H) = \begin{cases}
        d & d - 1 < x < d \\ x & \text{Otherwise}
    \end{cases}
    \]
    Therefore, for any $H = [c,d] \subseteq [a,b]$, if $d - c \ge 1$, we have $U (H) = 1 - \Pr_{x \sim D} \left[ f \le x < h \right]$, and if $d - c < 1$, we have
    \[
    U (H) = \begin{cases}
    1 - \Pr_{x \sim D} \left[ d-1 < x <  f \right] & d-1 < f \\
    1 - \Pr_{x \sim D} \left[ f \le x \le d-1 \right] & d-1 \ge f
    \end{cases}
    \]
    This is because under $d-c \ge 1$, no one manipulates, and thus, the error corresponds to the probability mass between $f$ and $h$: the positives who cannot manipulate to pass $h$. Under $d-c <1$, because every agent $x > d - 1$ can receive positive classification by manipulating, the error of $H$ corresponds to the probability mass between $d-1$ and $f$: if $d-1 < f$, this corresponds to the negatives who can manipulate and receive positive classification, and if $d-1 \ge f$, this corresponds to the positives who cannot manipulate to receive positive classification.
    
    Now assume $b - a < 1$, which implies that $d-c < 1$. At a high level, to maximize $U (H)$ in this case, we want to pick $d$ such that $d-1$ is as close as possible to $f$. More formally, our goal is to pick $[c,d] \subseteq [a,b]$ such that $h \in [c,d]$ and that the probability mass between $d-1$ and $f$ is minimized. In this case, one can see, via a case analysis, that $d = \min \left( b, \max \left( h, f+1 \right) \right)$ is the optimal value, and that $c$ can be any point in $[a, h]$.
    
    If $b - a \ge 1$, then both $d - c < 1$ and $d - c \ge 1$ are possible. If $d - c < 1$, then the optimality of $[c,d]$ where $c$ is any point in $(b - 1, h]$, and $d = \min \left( b, \max \left( h, f+1 \right) \right)$ can be established as above. Note that the choice of $c \in (b - 1, h]$ guarantees that $d - c < 1$. If $d - c \ge 1$, then the utility of the learner doesn't change if $[c,d] = [a, b]$ simply because the agents do not manipulate for any $[c,d]$ such that $d - c \ge 1$. Finally, the optimal interval is chosen  based on which case ($d - c <1$ vs. $d - c \ge 1$) leads to higher utility.
\end{proof}

\subsection{An Efficient Algorithm for Discrete Uniform Priors}\label{subsec:unif-prior}
In this section we will provide an efficient algorithm for computing the learner's optimal information release when the prior $\pi$ is a \emph{discrete} uniform distribution over a set $\{h_1, h_2, \ldots, h_n \} \subseteq \mathcal{H}$ that includes the adopted classifier $h$. The objective of the learner is to release a $H \subseteq \{h_1, h_2, \ldots, h_n \}$ such that $h \in H$. Throughout, we take $h = h_k$ where $1 \le k \le n$, and assume $h_1 \le h_2 \le \ldots \le h_n$. All the missing proofs of this section are provided in Appendix~\ref{app:b}.

We first state some facts about the agents' best response for \emph{any} prior $\pi$ over $\{h_1, h_2, \ldots, h_n \} \subseteq \mathcal{H}$. To start, we first show that the best response of any agent can be characterized as follows:

\begin{restatable}{lemma}{bestreponsecharac}\label{lem:bestreponse-charac}
    For any agent $x$, and any prior $\pi$ over $\{h_1, h_2, \ldots, h_n \} \subseteq \mathcal{H}$, we have 
    $
    \br (x,H) \in \{x\}\cup \{h_i \in H :h_i>x\}.
    $
\end{restatable}

This lemma basically tells us that the best response of any agent $x$ is either to stay at its location, or to manipulate to $h_i\in H$ such that  $h_i>x$. Given such characterization of the agents' best response in our setup, we now characterize, for any classifier $h_i$ in the support of $\pi$, the set of agents that will manipulate to $h_i$.

\begin{restatable}{lemma}{interval}\label{lem:interval}
    Fix any prior $\pi$ over $\{ h_1, \ldots, h_n \}$ and any $H$. If for any $i$, $\left\{ x: \br ( x, H) = h_i \right\} \neq \emptyset$, then for some $\alpha$,
    $\left\{ x: \br ( x, H) = h_i \right\} = (\alpha, h_i]
    $, where $\alpha$ satisfies $u_\alpha (\alpha, H) = u_\alpha(h_i, H)$.
\end{restatable}

Next, we characterize the utility of any subset $H$ released by the learner using a real-valued function of $H$. Define, for any $H \subseteq \{ h_1, \ldots, h_n\}$ such that $h \in H$,
\begin{equation}\label{eq:R_H}
R_H \triangleq \inf \left\{ x: \br ( x, H) \ge h \right\}
\end{equation}
Note that $\br (x=h, H) \ge h$ for any $H$ such that $h \in H$. Therefore, $\left\{ x: \br ( x, H) \ge h \right\}$ is nonempty, and that $R_H \le h$ for any $H$ such that $h \in H$. Our next lemma shows that $R_H$ characterizes the utility of $H$ for the learner, for any prior $\pi$ over $\{ h_1, \ldots, h_n \}$.

\begin{lemma}[Learner's Utility]\label{lem:learnersutility}
Fix any prior $\pi$ over $\{ h_1, \ldots, h_n \}$. We have that for any $H \subseteq \{ h_1, \ldots, h_n \}$ such that $h \in H$, the utility of the learner, given by Equation~\ref{eq:utilityPartial}, can be written as
\begin{equation}\label{eq:utility}
U (H) = \begin{cases}
    1 - \Pr_{x \sim D} \left[ R_H < x <  f \right] & R_H < f \\
    1 - \Pr_{x \sim D} \left[ f \le x \le R_H \right] & R_H \ge f
\end{cases}
\end{equation}
\end{lemma}

\begin{proof}[Proof of Lemma~\ref{lem:learnersutility}]
    Recall that
    $
    U (H) = \Pr_{x \sim D} \left[ h(\br(x,H)) = f(x)\right]
    $.
    The claim follows from the fact that $h(\br(x,H)) = 1$ if and only if $x > R_H$. Note that if $x > R_H$, then $\br(x,H) \ge h$ (equivalently, $h(\br(x,H)) = 1$) by the definition of $R_H$ and Lemma~\ref{lem:facts}. Further, if $\br(x,H) \ge h$ then $x > R_H$ by the definition of $R_H$.
\end{proof}

Given such characterization of the learner's utility, we will show that when the agents' prior is uniform over $\{ h_1, \ldots, h_n \}$, there are only \emph{polynomially many} possible values that $R_H$ can take, even though \emph{there are exponentially many $H$'s}. Our algorithm then for any possible value $R$ of $R_H$, finds a subset $H$ such that $R_H = R$, if such $H$ exists. The algorithm then outputs the $H$ with maximal utility according to Equation~\ref{eq:utility}. More formally, we consider the following partitioning of the space of subsets of $\{ h_1, \ldots, h_n \}$. For any $\ell \in \{ 1, 2, \ldots, n \}$, and for any $i \in \{ k, k+1, \ldots, n\}$\footnote{Recall $k$ is the index of $h$ in $\{ h_1, \ldots, h_n\}$, i.e., $h= h_k$.}, define
\[
S_{i,\ell} = \left\{ H \subseteq \left\{ h_1, \ldots, h_n \right\}: h \in H, \, | H | = \ell, \, \br (h, H) = h_i \right\}
\]
Note that by Lemma~\ref{lem:bestreponse-charac}, $\br (h \equiv h_k, H) \in \{ h_i : i \ge k \}$ for any $H$. Therefore, $\{S_{i,\ell}\}_{i,\ell}$ gives us a proper partitioning of the space of subsets, which implies
\[
\min_{H \subseteq \left\{ h_1, \ldots, h_n \right\}, h \in H} U (H) = \min_{i, \ell} \min_{H \in S_{i,\ell}} U (H)
\]
We will show that when the prior is uniform, solving $\min_{H \in S_{i,\ell}} U(H)$ can be done efficiently, by showing a construction of the optimal $H \in  S_{i,\ell}$ in our algorithm. To do so, we first show that $R_H$ (defined in Equation~\ref{eq:R_H}) can be characterized by $h_i$, when we restrict ourselves to $H \in S_{i,\ell}$.

\begin{restatable}{lemma}{hi}\label{lem:h_i}
   Fix any prior $\pi$ over $\{ h_1, \ldots, h_n \}$. If $H \in S_{i,\ell}$, then $\left\{ x: \br ( x, H) = h_i \right\} = (R_H, h_i]$.
\end{restatable}

In particular, this Lemma implies that for $H \in S_{i,\ell}$, we have $R_H = \inf \left\{ x: \br ( x, H) = h_i \right\}$. Next, we demonstrate the possible values that $R_H$ can take for uniform priors. In particular, the following lemma establishes that $R_H$ can take only polynomially many values.
\begin{restatable}{lemma}{RH}\label{lem:R_H}
    If the prior $\pi$ is uniform over $\{h_1, \ldots, h_n\}$, then for any $H \in S_{i,l}$,
    $
    R_H = h_i - j/\ell$ where $j = \left| \left\{ h' \in H:  h' \in (R_H, h_i] \right\} \right|
    $.
\end{restatable}

\begin{algorithm}[t]
    \SetAlgoNoLine
    \KwIn{ground truth classifier $f$, adopted classifier $h \ge f$, prior's support $\{h_1, \ldots, h_n\}$ where $h_1 \le \ldots \le h_n$ and $h_k = h$, data distribution $D$}
    \For{$i = k, k+1, \ldots, n$}{
        \For{$\ell = 1, 2, \ldots, n$}{
            \For{$j = 1, 2, \ldots, \ell$}{
                $R \gets h_i - j/\ell$ \tcp*{candidate value $R$ for $R_H$.}
                $S_1 \gets \left\{ h' \in \{h_1, \ldots, h_n\}: R < h' \le h_i \right\}$ \tcp*{all classifiers between $R$ and $h_i$.}
                $S_2 \gets \left\{ h' \in \{h_1, \ldots, h_n\}: h' \le R \right\}$ \tcp*{all classifiers smaller than $R$.}
                $S_3 \gets \left\{ h' \in \{h_1, \ldots, h_n\}: h' > h_i \right\}$ \tcp*{all classifiers larger than $h_i$.}
                \If{$R \ge h$ or $\left|  S_1 \right| < j$}{
                    $H_j^{i,\ell} \gets \perp$ \tcp*{no $H$ exists for $(i,\ell,j)$.}
                }\Else{
                    $H_j^{i,\ell} \gets \{h, h_i\}$\;
                    $H_j^{i,\ell} \gets H_j^{i,\ell} \cup \text{MAX}_{j-|H_j^{i,\ell}|} \left(S_1 \setminus H_j^{i,\ell} \right)$ \tcp*{$\text{MAX}_{m} (\cdot) \triangleq $ $m$ largest elements}
                    \If{$|S_2| \ge \ell - j$}{
                        $T \gets \text{any subset of size $\ell - j$ from $S_2$}$
                    }\Else{
                        $T \gets S_2 \cup \text{MAX}_{\ell-j - |S_2|} \left( S_3 \right)$ \tcp*{$\text{MAX}_{m} (\cdot) \triangleq $ $m$ largest elements}}
               
                    $H_j^{i,\ell} \gets H_j^{i,\ell} \cup T$\;
                    \If{$\br(h_i, H_j^{i,\ell}) > h_i$}{
                        $H_j^{i,\ell} \gets \perp$ \tcp*{no $H$ exists for $(i,\ell,j)$.}
                    }\Else{
                        \If{$\inf \left\{ x: \br ( x, H_j^{i,\ell}) = h_i \right\} > R$}{
                            $H_j^{i,\ell} \gets \perp$ \tcp*{no $H$ exists for $(i,\ell,j)$.}}
                          }
                    }
                
                \If{$H_j^{i,\ell} = \perp$}{
                    $U_j^{i, \ell} \gets -\infty$\;}
                \Else {
                    \If{$R < f$}{
                        $U_j^{i,\ell} \gets 1 - \Pr_{x \sim D} \left[ R < x <  f \right]$ \tcp*{computing utility according to Equation~\ref{eq:utility}.}}
                    
                    \If{$R \ge f$}{
                        $U_j^{i,\ell} \gets 1 - \Pr_{x \sim D} \left[ f \le x \le R \right]$ \tcp*{computing utility according to Equation~\ref{eq:utility}.}}
            	
                    }
                
            }
        }
   }
    \KwOut{$H^\star = H_{j^\star}^{i^\star, \ell^\star}$ where $(i^\star, \ell^\star, j^\star) \in \argmax_{(i,\ell,j)} U_j^{i, \ell}$.}
   
\caption{The Learner's Optimization Problem: Discrete Uniform Prior}
\label{alg:firm}
\end{algorithm}

Given such characterization, Algorithm~\ref{alg:firm}, for any $i, \ell$, enumerates over all possible $j \in \{1, \ldots, \ell\}$ and returns a $H$ such that $H \in S_{i, \ell}$ and $R_H = h_i - j / \ell$, if such $H$ exists. To elaborate, for any $i,\ell,j$, such $H \equiv H_{j}^{i,\ell}$ is constructed by first picking the $j$ largest classifiers that are between $h_i - j/\ell$ and $h_i$ (including both $h_i$ and $h$). If there are not at least $j$ classifiers between $h_i - j/\ell$ and $h_i$, then no such $H$ exists for $i,\ell,j$ because of Lemma~\ref{lem:R_H}. After picking the first $j$ elements as described, the remaining $\ell - j$ classifiers are first chosen from all classifiers that are less than (or equal to) $h_i - j/\ell$, and once these classifiers are exhausted, the rest are taken from the classifiers that are larger than $h_i$, starting from the largest possible classifier, and going down until $\ell$ classifiers are picked.

Note that this construction of $H \equiv H_{j}^{i,\ell}$ guarantees that $\br(h_i, H')$ is minimized among all $H'$'s with corresponding values of $(i,\ell,j)$. Therefore, if $\br(h_i, H) > h_i$, it is guaranteed that no $H$ exists for $(i,\ell,j)$. If $\br(h_i, H) = h_i$, then the construction of $H \equiv H_{j}^{i,\ell}$ guarantees that $R_H = \inf \left\{ x: \br ( x, H) = h_i \right\}$ is as small as possible. Therefore, if $\inf \left\{ x: \br ( x, H) = h_i \right\} > h_i - j /\ell$, then it is guaranteed that no such $H$ exists for $(i,\ell,j)$ (note that $\inf \left\{ x: \br ( x, H) = h_i \right\} \ge h_i - j /\ell$ by construction). The algorithm finally outputs, among all $H$'s found, the subset $H$ with maximum utility according to Equation~\ref{eq:utility}.

This proves the following theorem.

\begin{theorem}
    There exists an algorithm (Algorithm~\ref{alg:firm}) that for any $n$, any uniform prior over $\{h_1, \ldots, h_n\}$ that includes $h$, and any data distribution $D$, returns $H^\star \subseteq \{h_1, \ldots, h_n\}$ in time $O(n^3)$ such that $h \in H^\star$, and that $U (H^\star) = \max_{H \subseteq \{h_1, \ldots, h_n\}, h \in H} U (H)$.
\end{theorem}

\subsection{Minimizing False Positive (Negative) Rates for Arbitrary Priors}
While so far we worked with \emph{accuracy} as the utility function of the learner, in this section, we consider other natural performance metrics and provide insights on the optimal information release for the proposed utility functions, without restricting ourselves to uniform priors. In particular, we consider \textit{False Negative Rate} (FNR) and \textit{False Positive Rate} (FPR) which are crucial metrics to consider in  hiring, loan approvals, and many other applications. FNR represents the proportion of actual positive cases that are incorrectly predicted as negative, while FPR represents the proportion of actual negative cases that are incorrectly predicted as positive by the classifier. We formally define these new utility functions for the learner below. For any $H \subseteq \mathcal{H}$ such that $h \in H$,
\begin{equation}
U_{FPR} (H) \triangleq 1 - FPR (H) \triangleq 1 - \Pr_{x\sim D} \left[h(BR(x,H))=1|f(x)=0 \right]
\end{equation}
\begin{equation}
U_{FNR} (H) \triangleq 1 - FNR (H) \triangleq 1 - \Pr_{x\sim D}\left[ h(BR(x,H))=0|f(x)=1 \right]
\end{equation}
Throughout this section, instead of maximizing the utilities $U_{FPR}$ or $U_{FNR}$, we use the equivalent phrasing of minimizing $FPR$ or $FNR$, respectively. We note that we implicitly assume that $\Pr_{x\sim D}[f(x)=1]>0$ and $\Pr_{x\sim D}[f(x)=0]>0$ --- otherwise, the learner can simply achieve perfect utility by rejecting or accepting all agents, respectively.

In the following theorem, we establish that for any given prior $\pi$ over a set $\{ h_1, h_2, \ldots, h_n\} \subseteq \mathcal H$, if the learner aims to minimize the FPR, \emph{no-information-release} is preferable to \emph{full-information-release}.\footnote{We note that in this section, while we work with discrete priors over some $\{h_1, \ldots, h_n\} \subseteq \mathcal H$, our results can be easily extended to \emph{any} prior.} Additionally, we show that for minimizing the FNR, an optimal strategy for the learner is \emph{full-information-release}. By ``no-information-release'', we mean releasing any subset $H$ such that $H$ includes the support of the prior $\pi$: $H \supseteq \{ h_1, \ldots, h_n \}$ which results in $\pi |_H = \pi$; one such $H$ could be the entire class $H = \mathcal H$. By ``full-information-release'', we mean revealing the classifier: $H = \{ h \}$.

\begin{theorem}\label{thm:fpr-fnr}
Fix any $h \ge f$. For any prior $\pi$ over $\{h_1, \ldots, h_n\}$ that includes $h$, we have
\begin{enumerate}
    \item $FPR(\mathcal H)\leq FPR(\{h\})$.
    \item $FNR(\{h\})\leq FNR(H)$ for every $H\subseteq \mathcal H$ such that $h \in H$.
\end{enumerate}
\end{theorem}

\begin{proof}[Proof of Theorem~\ref{thm:fpr-fnr}]
We begin by showing that $FPR(\mathcal H)\leq FPR(\{h\})$. Let $x\in \mathcal{X}$ be such that $f(x)=0$ and $h(BR(x,\mathcal H))=1$. We will show that $h(BR(x,\{h\}))=1$.

Lemma~\ref{lem:bestreponse-charac} together with $h\geq f$ imply the existence of $h_{j}$ such that $BR(x,\mathcal H)=h_{j}>x$ (as $f(x)\ne h(\br(x,\mathcal{H}))$). This further indicates that when $\mathcal{H}$ is released, the utility of the agent is strictly higher when it manipulates to $h_j$, compared to not moving:
\[
    u_x (h_j, \mathcal{H}) = \sum_{i=1}^{j} \pi (h_{i})-(h_{j}-x)> \sum_{i: h_i \le x } \pi (h_{i}) = u_x (x, \mathcal{H})
\]
Note that $h(BR(x,\mathcal H))=1$ and $BR(x,\mathcal H) = h_j$ implies that $h_{j}\geq h$, and therefore:
\[
    u_x (h, \{h\}) = 1-(h-x)\geq \sum_{i=1}^{j} \pi (h_{i})-(h_{j}-x)> \sum_{i: h_i \le x} \pi (h_{k}) = u_x(x, \{h\})
\]
Since Lemma~\ref{lem:bestreponse-charac}   implies that $BR(x,\{h\})\in \{x,h\}$, we derive from the above inequality that $h(BR(x,\{h\}))=1$. This proves the first part of the theorem. 

Next, we show that $FNR(\{h\})\leq FNR(H)$ for every $H\subseteq \mathcal H$. 
Let $x\in \mathcal{X}$ be such that $f(x)=1$ and $h(BR(x,\{h\}))=0$, and let $H$ be any subset of $\mathcal{H}$. We will  show that $h(BR(x, H))=0$.

Lemma~\ref{lem:bestreponse-charac} implies that $BR(x,\{h\})\in \{x,h\}$. Together with  $h(BR(x,\{h\}))=0$, we derive that $BR(x,\{h\})=x$, and thus:
\[
    u_x (h, \{h\}) = 1-(h-x) \leq u_x (x, \{h\}) = 0
\]
Now, for every $h_j$ such that $h_j\geq h$, we have:
\[    u_x (h_j, H) = \sum_{i=1}^{j} \pi|_H(h_{i})-(h_{j}-x) \leq 1-(h-x) \leq 0
    \leq \sum_{i: h_i \le x } \pi |_H(h_{i}) = u_x (x, H).
\]
As a result, when the learner releases $H$, the utility of agent $x$ from remaining at $x$ is greater than (or equal to) any manipulation $h_j$ such that $h_j\geq h$. This implies that $h(BR(x,H))=0$.
\end{proof}

We finish this section by showing that minimizing FPR, unlike minimizing FNR, does not always have a clear optimal solution, by providing three examples with very different optimal solutions.

\begin{example}[Full-information-release is optimal for FPR] Fix any $B > 1$ and any $0 \le t < B -1$. Let $D$ be the uniform distribution over $\mathcal{X} = [0,B]$, and $f(x) = \mathds{1} \left[ x \ge t \right] $. Let $\mathcal H$ be the class of single-sided threshold classifiers and suppose the adopted classifier $h(x) = \mathds{1} \left[ x \ge t+1 \right]$. Under any prior over $\mathcal H$, one can show that the full information release of $H = \{h\}$ achieves perfect FPR for this setting: $FPR(\{h\}) = 0$.
\end{example}

\begin{example}[No-information-release is optimal for FPR] Under the same setup as in Example~\ref{ex:partailCanbeBetter2}, one can show that releasing the support of the prior $H = \{ h_{1.8}, h_2 \}$ achieves $FPR(H) = 0$, whereas full information release of the adopted classifier $h=h_2$ achieves $FPR(\{ h \}) = 0.9/1.9 \approx 0.47$. Note that $H = \{ h_{1.8}, h_2\}$ is the support of the prior, so it constitutes as no-information-release. In other words, we have $FPR(H') = FPR(H) = 0$ for every $H'$ such that $H\subseteq H' \subseteq \mathcal H$.
\end{example}


\begin{claim}\label{claim:fpr}
    There exists an instance in which neither full-information-release nor no-information-release are optimal when the utility function of the learner is $U_{FPR}$.\footnote{We remark that the claim holds when the utility function is $U$ (as defined in Equation~\ref{eq:utilityPartial}) as well.}
\end{claim}

\begin{proof}[Proof of Claim~\ref{claim:fpr}] 
We construct such an instance as follows. Suppose the domain is $\mathcal X=\{x_1,x_2\}$ with  $x_1=0$, $x_2=0.4$ and the distribution $D$ is given by $D(x_1)=D(x_2)= 0.5$. 
In addition, consider $f=0.3$, and hypothesis class $\mathcal{H}=\{h_1,h_2,h_3\}$ where $h_1=0.1,h_2=0.5,h_3=0.7$, and a prior distribution $\pi$ such that $\pi(h_1)=0.2,\pi(h_2)=0.1,\pi(h_3)=0.7$. 

Observe that under full-information-release, $FPR( \{h\})=1$ for every $h\in \mathcal{H}$. Now suppose $h=h_2$ is the adopted classifier. 
We have that 
    $BR(x_1, \{h\})=0.5$ implying $h(BR(x_1, \{h\}))=1\ne f(x_1) = 0$ implying $x_1$ is a false positive under $\{h\}$ release. Additionally, $BR(x_1, \mathcal{H})=0.7$ implies that $h(BR(x_1, \mathcal{H}))=1\ne f(x_1) = 0$ implying $x_1$ is a false positive under $\mathcal{H}$ release.
    Further, it holds that $BR(x_1, \{h_1,h_2\})=0.1$ and so $h(BR(x_1, \{h_1, h_2 \}))=0=f(x_1)$. Moreover, in this particular instance, releasing $\{h_1,h_2\}$ achieves perfect utility as $BR(x_2, \{h_1,h_2\})=0.5$ which implies $h(BR(x_2, \{h_1, h_2\}))=1=f(x_2)$.
\end{proof}

\section*{Conclusion}
We initiate the study of strategic classification with partial knowledge (of the agents) and partial information release (of the learner). Our model relaxes the often unrealistic assumption that agents fully know the learner's deployed classifier. Instead, we model agents as having a distributional prior on which classifier the learner is using. Our results show the existence of previously unknown intriguing informational middle grounds; they also demonstrate the value of revisiting the fundamental modeling assumptions of strategic classification in order to provide effective recommendations to practitioners in high-stakes, real-world prediction tasks. 

A critical line for future research is how our partial information release model can be used to achieve \emph{fairness} in strategic classification. Prior work studying fairness in the standard strategic classification setting \citep{hu2019disparate, socialcost18} consider population groups that have \emph{differing cost functions}. In our model of strategic classification, a critical aspect that requires further investigation is designing fairness-aware information release when different population groups not only have differing cost functions but they can also have \emph{differing prior distributions}: network homophily, social disparities, and stratification will cause population groups to have different priors.
Can the learner maintain strategic accuracy while reducing disparities by using \emph{fairness-aware} partial information release to agents?
We believe our model can help shed light on this critical research problem.

\section*{Acknowledgements}

The authors thank Avrim Blum for helpful discussions in the early stages of this work.

Lee Cohen is supported by   the Simons Foundation Collaboration on the Theory of Algorithmic Fairness, the Sloan Foundation Grant 2020-13941, and the Simons Foundation investigators award 689988.
Kevin Stangl was supported in part by the National Science Foundation under grants CCF-2212968 and ECCS-2216899, by the Simons Foundation under the Simons Collaboration on the Theory of Algorithmic Fairness, and by the Defense Advanced Research Projects Agency under cooperative agreement HR00112020003. The views expressed in this work do not necessarily reflect the position or the policy of the Government and no official endorsement should be inferred.

\bibliographystyle{ACM-Reference-Format}
\bibliography{main_arxiv}

\newpage
\appendix
\section{Missing Proofs of Section~\ref{sec:agentbr}}\label{app:c}

\submodular*

\begin{proof}[Proof of Claim~\ref{clm:submodular}]
We will abuse notation and use $h_i$ for the threshold $t_i$ ($h_i \equiv t_i \in \reals$).

Fix any $x$. Consider $S \subseteq S' \subseteq \{h_1, \ldots, h_n \}$ and $h' \in \reals$ such that $h' \notin S'$. Note that
\[
c(x,S) = \max \left( \max (S) - x, 0 \right), \quad c \left(x, S\cup\{h'\} \right) = \max \left( \max (S\cup\{h'\}) - x, 0 \right)
\]
\[
c(x,S') = \max \left( \max (S') - x, 0 \right), \quad c \left(x, S'\cup\{h'\} \right) = \max \left( \max (S'\cup\{h'\}) - x, 0 \right)
\]
where $\max (F)$ is simply the largest threshold in $F$, for any set $F$. Note that $\max (S) \le \max (S')$ because $S \subseteq S'$. Suppose $x \le \max (S)$. We have three cases
\begin{enumerate}
    \item If $h' \ge \max (S')$, then
    \[
    c \left(x, S\cup\{h'\} \right) - c \left(x, S \right) = h' - \max(S) \ge h' - \max(S') = c \left(x, S'\cup\{h'\} \right) - c \left(x, S' \right)
    \]
    \item If $\max(S) \le h' \le \max(S')$, then
    \[
    c \left(x, S\cup\{h'\} \right) - c \left(x, S \right) = h' - \max(S) \ge 0 = c \left(x, S'\cup\{h'\} \right) - c \left(x, S' \right)
    \]
    \item If $h' \le \max(S)$, then
    \[
    c \left(x, S\cup\{h'\} \right) - c \left(x, S \right) = c \left(x, S'\cup\{h'\} \right) - c \left(x, S' \right) = 0
    \]
\end{enumerate}
    So we have shown that the cost function is submodular if $x \le \max(S)$. We can similarly, using a case analysis, show that the cost function is submodular when $x > \max(S)$.
\end{proof}

\section{Missing Proofs of Section~\ref{subsec:setup}}\label{app:a}

\facts*

\begin{proof}[Proof of Lemma~\ref{lem:facts}]
    Fix any $x$, and any $H$. We have
    \begin{align*}
    \br(x, H) &= \argmax_{x'} \left\{ \Pr_{h' \sim \pi \vert_H} \left[x' \ge h' \right] - | x' - x| \right\}
    \\ &= \argmax_{x' \ge x} \left\{ \Pr_{h' \sim \pi \vert_H} \left[x' \ge h' \right] - ( x' - x) \right\}
    \\ &= \argmax_{x' \ge x} \left\{ \Pr_{h' \sim \pi \vert_H} \left[x' \ge h' \right] -  x' \right\}
    \\ & = \argmax_{x' \ge x} g_H (x')
    \end{align*}
    where we take $g_H (x') \triangleq \Pr_{h' \sim P \vert_H} \left[x' \ge h' \right] -  x'$. The first equality follows because agents don't gain any utility by moving to a point $x' < x$, and that tie-breaking is in favor of lowest cost solution.
    
    The first and the second part of the lemma follows from this derivation. For the third part, we have that
    \begin{align*}
        \br (x_1 , H) &= \argmax_{x' \ge x_1} g_H (x')
        = \argmax_{x' \ge x_2} g_H (x')
        = \br (x_2 , H)
    \end{align*}
    where the second equality follows because $\br (x_1 , H) \ge x_2$.
\end{proof}

\section{Missing Proofs of Section~\ref{subsec:unif-prior}}\label{app:b}

\bestreponsecharac*
\begin{proof}[Proof of Lemma~\ref{lem:bestreponse-charac}]
    Recall from Lemma~\ref{lem:facts} that $\br (x,H) \ge x$. Note that the utility of the agent $x\in \mathcal{X}$ from manipulating to a point $x'\geq x$ can be expressed as 
    \[
    u_x (x', H) = \sum_{i: h_i \le x'} \pi \vert_H(h_{i})-(x'-x)
    \]
    For any $x' \ge x$ such $x' \notin \{x\}\cup \{h_i \in H :h_i>x\}$, it is easy to see that the agent can increase her utility by moving to a point in $\{x\}\cup \{h_i \in H :h_i>x\}$, which proves the lemma.
\end{proof}

\interval*
\begin{proof}[Proof of Lemma~\ref{lem:interval}]
    Let $\alpha = \inf \left\{ x: \br ( x, H) = h_i \right\}$. Take any $x \in (\alpha, h_i]$. We have, by the definition of $\alpha$, that there exists $x' \in (\alpha, x)$ such that $x' \in \left\{ x: \br ( x, H) = h_i \right\}$, implying $\br ( x', H) = h_i$. Therefore, $\br ( x', H) \ge x$. The third part of Lemma~\ref{lem:facts} implies that $h_i = \br ( x', H) = \br ( x, H)$. This proves that
    \[
    (\alpha, h_i] \subseteq \left\{ x: \br ( x, H) = h_i \right\}
    \]
    
    If $x > h_i$, then $\br(x,H) > h_i$ by the first part of Lemma~\ref{lem:facts}. Therefore $x \notin \left\{ x: \br ( x, H) = h_i \right\}$.
    
    If $x < \alpha$, then $\br(x,H) < h_i$ by the definition of $\alpha$, implying $x \notin \left\{ x: \br ( x, H) = h_i \right\}$.
    
    If $x=\alpha$, we will show that $\br(x, H) = \alpha$. Note that $\alpha \le \br(\alpha, H) \le h_i$ by Lemma~\ref{lem:facts}. But if $\br(\alpha, H) > \alpha$, then by the third part of Lemma~\ref{lem:facts}, $\br(\alpha, H) = h_i$. So $\br(\alpha, H) \in \{ \alpha, h_i \}$. Suppose $\br(\alpha, H) = h_i$. Therefore, there exists $\epsilon > 0$ such that $u_\alpha (\alpha, H) + \epsilon < u_\alpha (h_i, H)$, by the definition of agents' best response and the fact that tie-breaking is in favor of smaller values (Remark~\ref{remark:tie-breaking}). Let $\epsilon' \in (0, \epsilon/2]$ be such that $\{ h' \in H: \alpha - \epsilon' \le h' < \alpha\} = \emptyset$. Consider $x' = \alpha - \epsilon'$. We have, by Lemma~\ref{lem:facts} and ~\ref{lem:bestreponse-charac}, that $\br(x', H) \in \{x' \} \cup [\alpha, h_i]$. But because $\br(\alpha, H) = h_i$, Lemma~\ref{lem:facts} implies that $\br(x', H) \in \{x', h_i\}$. Note that
    \[
    u_{x'} (x', H) \le u_\alpha (\alpha, H) < u_\alpha (h_i, H) - \epsilon = u_{x'} (h_i, H) - ( \epsilon - \epsilon')
    \]
    implying that $\br(x', H) = h_i$. But $x' = \alpha - \epsilon'$ and this contradicts with the definition of $\alpha$. Therefore $\br(\alpha, H) = \alpha$, and this completes the proof of the first part of the lemma.
    
    We now focus on the second part of the lemma. Note that $u_\alpha (\alpha, H) \ge u_\alpha (h_i, H)$, because if $u_\alpha (\alpha, H) < u_\alpha (h_i, H)$, then $\alpha < \br(\alpha, H) \le h_i$. Together with the first part of this lemma, and Lemma~\ref{lem:facts}, this implies that $\br(\alpha, H) = h_i$ which is a contradiction with the first part of the lemma. Next, we show that $u_{\alpha} (\alpha, H) \le u_{\alpha} (h_i, H)$. Suppose $u_{\alpha} (\alpha, H) > u_{\alpha} (h_i, H) + \epsilon$ for some $\epsilon > 0$. Consider $x = \alpha + \epsilon / 2$. We have that
    \[
    u_x (x, H) \ge u_{\alpha} (\alpha, H) > u_{\alpha} (h_i, H) + \epsilon = u_{x} (h_i, H) + \epsilon/2
    \]
    implying that $\br(x,H) \neq h_i$. This is in contradiction with the first part of the lemma. Therefore, $u_\alpha (\alpha, H) = u_\alpha (h_i, H)$.
\end{proof}

\hi*
\begin{proof}[Proof of Lemma~\ref{lem:h_i}]
    We first show that $R_H = \inf \left\{ x: \br ( x, H) = h_i \right\}$. Fix any $H \in S_{i,\ell}$. Let $Q_H = \inf \left\{ x: \br ( x, H) = h_i \right\}$. First note that because $H \in S_{i,\ell}$, we have $\br (h, H) = h_i$, and therefore $\left\{ x: \br ( x, H) = h_i \right\} \neq \emptyset$, and that $Q_H \le h$. Additionally, because $$\left\{ x: \br ( x, H) = h_i \right\} \subseteq \left\{ x: \br ( x, H) \ge h \right\}$$
    we have that $Q_H \ge R_H$. So we have $R_H \le Q_H \le h$. If $Q_H \neq R_H$, then there exists $R_H < x < Q_H$, such that $h \le \br (x, H) < h_i$. But, for $x' = \br (x, H)$, we have $\br (x', H) = h_i > x' = \br (x, H)$. This is in contradiction with the third part of Lemma~\ref{lem:facts} (taking $x_1 = x$, and $x_2 = x'$). Therefore, $Q_H = R_H$, and this proves the first part of the lemma. The second part of the lemma is followed from part one and Lemma~\ref{lem:interval}.


    

\end{proof}

\RH*
\begin{proof}[Proof of Lemma~\ref{lem:R_H}]
    Fix any $H \in S_{i,\ell}$. Note that Lemma~\ref{lem:h_i} and Lemma~\ref{lem:interval} together imply that $u_{R_H} (R_H, H) = u_{R_H} (h_i, H)$. This implies
    \[
    \Pr_{h' \sim \pi \vert_H} \left[R_H \ge h' \right] = \Pr_{h' \sim \pi \vert_H} \left[h_i \ge h' \right] + (h_i - R_H)
    \]
    But $\Pr_{h' \sim \pi \vert_H} \left[R_H \ge h' \right] = j_1/\ell$ and $\Pr_{h' \sim \pi \vert_H} \left[R_H \ge h' \right] = j_2/\ell$ where $j_1$ and $j_2$ are the number of hypotheses in $H$ that are smaller (or equal to) $R_H$, and smaller (or equal to) $h_i$, respectively. In other words,
    \[
    j_1 = \left| \left\{ h' \in H: h' \le R_H \right\} \right|, \quad j_2= \left| \left\{ h' \in H: h' \le h_i \right\} \right|
    \]
    Therefore,
    \[
    R_H = h_i - \frac{j_2-j_1}{\ell}
    \]
    which completes the proof.
    
\end{proof}

\end{document}